\documentclass[runningheads]{llncs}
\usepackage[T1]{fontenc}
\usepackage{amsmath}
\usepackage{amssymb}
\usepackage{tablefootnote}
\usepackage{booktabs}
\usepackage{subfig}
\usepackage{multirow}
\usepackage{mathtools}
\usepackage{longtable}
\usepackage{graphicx}
\usepackage{hyperref}
\usepackage[capitalise]{cleveref}
\usepackage[linesnumbered,ruled]{algorithm2e}
\usepackage{setspace}
\usepackage{float}
\usepackage{multicol}
\usepackage{stfloats}
\usepackage{orcidlink} 
\hypersetup{hidelinks}

\usepackage{color}
\usepackage[misc]{ifsym} 
\begin{document}
\title{CNMBI: Determining the Number of Clusters Using Center Pairwise Matching and Boundary Filtering}
\titlerunning{CNMBI}
\author{Ruilin Zhang\textsuperscript{1}\orcidlink{0000-0002-4818-9282} \and
Haiyang Zheng\textsuperscript{1}\orcidlink{0000-0001-8733-9696} \and
Hongpeng Wang\textsuperscript{1(\Letter),2}\orcidlink{0000-0001-8108-2674}}
\authorrunning{Ruilin Zhang et al.}
%
\institute{Harbin Institute of Technology, Shenzhen, Shenzhen, China \\ \email{wanghp@hit.edu.cn} \and
 Peng Cheng Laboratory, Shenzhen, China \\}

\maketitle              
\begin{abstract}
	One of the main challenges in data mining is choosing the optimal number of clusters without prior information. Notably, existing methods are usually in the philosophy of cluster validation and hence have underlying assumptions on data distribution, which prevents their application to complex data such as large-scale images and high-dimensional data from the real world. In this regard, we propose an approach named CNMBI. Leveraging the distribution information inherent in the data space, we map the target task as a dynamic comparison process between cluster centers regarding positional behavior, without relying on the complete clustering results and designing the complex validity index as before. Bipartite graph theory is then employed to efficiently model this process. Additionally, we find that different samples have different confidence levels and thereby actively remove low-confidence ones, which is, for the first time to our knowledge, considered in cluster number determination. CNMBI is robust and allows for more flexibility in the dimension and shape of the target data (e.g., CIFAR-10 and STL-10). Extensive comparison studies with state-of-the-art competitors on various challenging datasets demonstrate the superiority of our method.

	\keywords{Number of clusters  \and Cluster center \and Complex data \and Pairwise matching \and Boundary filtering.}
\end{abstract}

\section{Introduction}
The automatic determination of parameters has become one of the main considerations for the popularity of a learning solution in machine learning or pattern recognition\cite{Contrastive}. Clustering, as an essential analysis tool, endeavors to discover underlying structure from unlabeled data, ensuring that samples\footnote{Data points, objects, and samples are used exchangeably in this paper.} in the same group have high homogeneity and that different groups have the maximum difference\cite{Fast-LDP-MST}. It should be noted that, typically, most clustering approaches suffer from the limitation that the number of clusters has to be fed by a human user, which is fundamental to obtaining a good data partition. However, the number of clusters in practice is usually unknown, as Salvador stated in \cite{Salvador}, which often results in suboptimal clustering results. Determining the number of clusters automatically can enhance the potential of clustering methods, while also being instructive for other learning tasks, like counting the amount of farmland in remote sensing images and setting the type of anchors in object detection\cite{Index-Based-Solutions-for-Efficient-Density-Peak-Clustering}. Although several solutions\cite{LCCV,DPC-AHS,VCIM,OCVD} have recently been proposed to estimate the number of clusters, they either perform unstably or are challenging to employ in practice when handling complex data like large-scale images and high-dimensional data from the real world. Technically, existing methods are generally based on the clustering validation philosophy, which uses the quality of clustering results to reveal the number of clusters, i.e., evaluating a clustering validity index (named CVI) over the target data and optimizing it as a function of the number of clusters. With a clear semantic and mathematical background, this paradigm has gained significant attention\cite{CNAK,Determination-of-the-Optimal-Number-of-Clusters:-A-Fuzzy-Set-Based-Method}, mainly covering the design of cluster validity index (CVI)\cite{RDPC-DSSA,DPC-AHS,CNAK} and the embedding of clustering schemes\cite{LCCV,OCVD,Index-Based-Solutions-for-Efficient-Density-Peak-Clustering}.

Nevertheless, the unsupervised context limits existing methods to design CVI based solely on the subjective morphological characteristics of final clusters, such as compactness and separation, which inevitably imposes potential assumptions on the target data distribution. Moreover, the clustering results of each round during iterative optimization are sometimes not discriminative and instructive, as in the case of complex data, due to the inherent limitations of the chosen clustering scheme. Commonly used K-means, for instance, favors spherical or well-separated clusters, so the consistency of the optimal number of clusters, the best clustering result, and the best score, i.e., the philosophy underlying such methods, could be invalidated by some challenging data. Additionally, we observe that existing work generally focuses on the methodological level while ignoring the significance of the sample to the target task in practice.

With this in mind, we bridge these bottlenecks by proposing a novel algorithm (CNMBI). Our work demonstrates the following contributions:

1) We propose a scheme for determining the number of clusters based on center pairwise matching. Instead of previous approaches resorting to CVI scoring and complete clustering results, we use inherent distribution information in the data space to model cluster number determination as a dynamic contrastive process of cluster centers regarding positional behavior.

2) We develop a boundary filtering method to remove low-confidence samples that interfere with the cluster number determination task, allowing our method to have greater flexibility and robustness in terms of data shape and dimension. To the best of our knowledge, this aspect has not been considered.

3) Our method outperforms the state-of-the-art competitors on diverse benchmark datasets (e.g., STL-10, MNIST, and CIFAR-10), and we are the first to report the results of a cluster number determination method on the challenging dataset such as STL-10 and CIFAR-10. The in-depth discussions of boundary information, and robustness jointly demonstrate the superiority of CNMBI.


\section{Related Work}
For the goal of determining the number of clusters, the philosophy behind most methods is that the actual number of clusters, the best clustering result, and the best clustering validity score are consistent. In this context, designing an effective validity index is essential because the optimal number of clusters needs to be reflected by evaluating the clustering results.  Conventionally, most CVIs take distance information of resulting clusters as the main focus, considering the intra-cluster distance (compactness), inter-cluster distance (separation), or the statistical variants of both to describe clustering results, such as DBI\cite{DBI}, DUNN\cite{Dunn}, CH\cite{CH-index}. Combined with partition-based clustering, these indexes can suggest the desired number of clusters  under the spherical or well-separated data. Additionally, some work attempts to construct various curves formed by intra-cluster variance, using characteristics such as "maximum curvature"\cite{Curvature}, "knee of the curve"\cite{Salvador}, "elbow of the curve", "max magnitude"\cite{evaluation-graph-gap}, or "sharp jump"\cite{Jump} to visually locate the number of clusters. 

Recently, in addition to the work above relying on distance information, several efforts have leveraged density information to define CVI or update clustering schemes. For instance, \cite{OCVD} uses cluster density to evaluate the clustering results. Likewise, LCCV \cite{LCCV} selects a group of highly representative density cores by density information to simplify the silhouette coefficient and employs agglomerative hierarchical clustering as the assignment scheme. Besides, there are some promising density clustering methods that work with CVI, one of the typical paradigms is density peaks clustering (DPC)\cite{DPC}. As in the case of \cite{RDPC-DSSA}, which simplifies the adapted silhouette coefficient by considering density centers from DPC as representative. Another example is DPC-AHS\cite{DPC-AHS}, which takes the skew of distribution within the cluster as the validity index by the advantages of density clustering for arbitrary-shaped data. Despite their simplicity and ease of implementation, these methods are not practical as they usually perform unstably on some complex data such as large-scale images and high-dimensional data from the real world, as discussed in the introduction. In contrast, our method addresses the above limitations.


\section{Proposed Method: CNMBI}

\subsection{Center-Pairwise Matching}

Cluster number determination is essentially a dynamic optimization process, so one potential solution is to improve identification performance by leveraging valuable pattern information. However, in such an unsupervised context, significant external information, such as data labels, is unavailable. It should be noted that the distribution information from the data space is also instructive, such as the cluster center generated by density or geometric information, which is closely related to the number of clusters and, more importantly, is the most representative in the data space.

With this in mind, we find a fundamental insight that the distance between the density centers (using density information) and the mean centers (relying geometric information) is closest when and only when the number of clusters is actual, which is enlightening for cluster number determination. As shown in \cref{Fig-3cases}, the distances between the mean centers and the density centers are small and the positions almost overlap in the case of the number of clusters $k=3$. On the contrary, there is a significant distance between the two in cases smaller than the actual number of clusters (\cref{Fig-3cases}(a), $k$=2) or larger (\cref{Fig-3cases}(c), $k$=5).To further discuss the principle behind this observation above, we have some formal definitions as follows. 
\begin{figure}[h]
	\centering
	\subfloat[case 1: $k=2<K=3$.]
	{ {\includegraphics[width=0.32\linewidth]{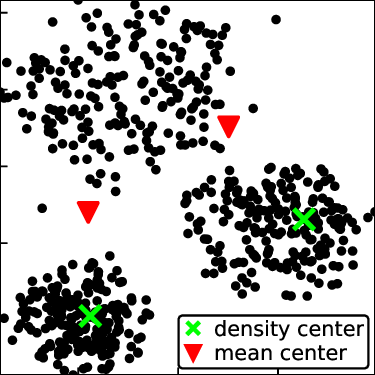}}}
	\subfloat[case 2: $k=K=3$.]
	{	 {\includegraphics[width=0.32\linewidth]{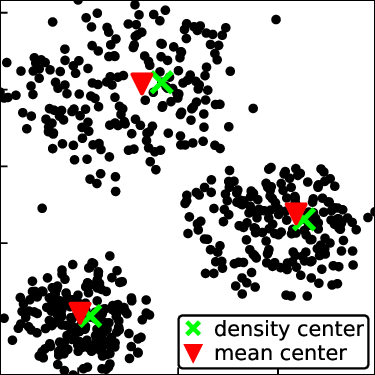}}} 
	\subfloat[case 3: $k=5>K=3$.]
	{	 {\includegraphics[width=0.32\linewidth]{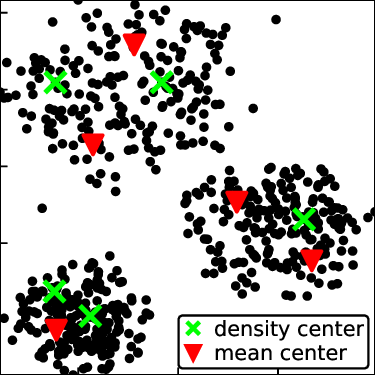}}}
	\setlength{\abovecaptionskip}{1.5mm} 
	\caption{Density center versus mean center for different cluster number settings. }
	\label{Fig-3cases}
\end{figure}

Given a dataset $X\in \mathbb{R}^d$ containing $K$ actual clusters of $n$ samples. Let $k$ denote the number of clusters currently under consideration, the mean center is obtained by
\begin{equation} \label{EQ1}
\begin{aligned}
S_{\mu }^{(k)}=\left\{ \left. \frac{\sum\limits_{x\in C_{\mu }^{(i)}}{x}}{\left| C_{\mu }^{(i)} \right|} \right|1\le i\le k \right\}
\end{aligned}
\end{equation}
where $S_{\mu }^{(k)}$ represents a set containing $k$ mean centers and $C_{\mu }^{(i)}$ denotes a cluster generated by partitioning-based clustering (K-means is used here), i.e. $X=\{C_{\mu }^{(1)} ,...,C_{\mu }^{(k)}\}$. In particular, \cref{EQ1}  illustrates that the mean centers are geometric centroids of the resulting clusters and their positions are unstable. As shown in \cref{Fig-3cases}, in the case of $k$=2 or $k=5$, the mean centers are generally distributed at the edges or outside of the actual clusters in $X$, due to cluster splitting or cluster merging behaviors.

In addition, the density centers are described as those objects surrounded by neighbors with lower density and have a relatively large distance from any data points with higher density, as discussed in DPC. To this aim, the two quantities: local density $\rho$ and $\delta$ distance are used to measure the possibility of each object becoming a cluster center from the density view.
\begin{equation} \label{EQ2}
\begin{aligned}
{{\rho }_{i}}=\sum\limits_{{{x}_{j}}\in X\backslash \{{{x}_{i}}\}}{\exp (-{{(\frac{{{d}_{ij}}}{dc})}^{2}})}
\end{aligned}
\end{equation}
\begin{equation} \label{EQ3}
{{\delta }_{i}}=\left\{ \begin{aligned}
& \min \{\left. dist({{x}_{i}},{{x}_{j}}) \right|{{\rho }_{i}}<{{\rho }_{j}},\ \ {{x}_{i}}\in X\}\ \ \ if\ \exists {{x}_{j}}\in X,\ \ {{\rho }_{i}}<\rho_j  \\ 
& \max \{\left. dist({{x}_{i}},{{x}_{j}}) \right|{{x}_{j}}\in X\}\ \ otherwise \\ 
\end{aligned} \right.
\end{equation}
Where $dc$ acts as the sampling radius of density calculation, it is determined by sorting $n^2$ distances in X from ascending order and selecting a number ranked at $2\%$. With the product ${{\gamma }_{i}}={{\rho }_{i}}\times{{\delta }_{i}}$, the greater the score ${{\gamma }_{i}}$, the more likely the ${{x}_{i}}$ is considered as a cluster center.

\begin{definition}\textbf{Density Center Set}  \label{Density Center Set}
	Given the current cluster number $k$, the density center set $S_{\gamma }^{(k)}$, including $k$ centers determined by DPC, is given by
\end{definition}
\begin{equation} \label{EQ4}
\begin{aligned}
S_{\gamma}^{(k)}=\{{{y}_{i}}\left| {{\gamma }_{i}}\ge {{{\bar{\gamma }}}_{k}} \right.,\ \ 1\le i\le n\}
\end{aligned}
\end{equation}
where ${{\bar{\gamma }}_{k}}$ is the $k$-th largest gamma score by sorting objects in $\gamma$ descending order. Note that the cardinality of both sets $S_{\gamma }^{(k)}$ and $S_{\mu }^{(k)}$ is $k$. According to the above formula, we can also conclude that the density center tends to be the most core region of each actual cluster in $X$, and the position is stable because the density information directly determines it.


Inspired by above observation, we draw a promising insight favoring the determination of the number of clusters: "the difference between the mean center and the density center in terms of location distribution is minimal when and only when the actual number of clusters is given, quantified as the sum of the $k$ pairwise distances formed is minimal compared to the other cases".


To apply this insight, we propose to model the relationship between the density centers and mean centers using graph theory to quantify the actual distance between them. The solution regarding the optimal number of clusters can be transformed into the searching the "Center Pairwise Matching" below.

We formally treat Mean Center Set $S_\mu$ (following \cref{EQ1}) and Density Center Set $S_\gamma$ (following \cref{EQ4}) as two graph node sets, and the elements of which are regarded as individual nodes. After that,  Center Bipartite Graph is defined as.
\begin{definition}\textbf{Center Bipartite Graph}\label{Center Bipartite Graph}
	Let $S_{\gamma }^{(k)}$ and $S_{\mu }^{(k)}$ be the set of mean center and density center at number of clusters $k$, respectively. We then define the center bipartite graph $G_{\mu ,\gamma }^{(k)}=({{V}^{(k)}},{{E}^{(k)}})$ such that it should obey the properties:
\end{definition}
(1) ${{V}^{(k)}}=\{S_{\mu }^{(k)}\cup S_{\gamma }^{(k)}\}$ and $S_{\mu }^{(k)}\cap S_{\gamma }^{(k)}=\emptyset$; (2) for each node $x_i\in S_{\mu }^{(k)}(or\ x_j\in \ S_{\gamma }^{(k)})$, there must be $k$ edges between it and another node subset; (3) for each edge $e(x_i,x_j)\in {{E}^{(k)}}$, that means $x_j\in S_{\mu }^{(k)}$ as well as $x_j\in S_{\gamma }^{(k)}$.


\begin{definition}\textbf{Center Complete Matching}\label{Center Complete Matching}
	Let $CM$ represent a complete matching on $G_{\mu ,\gamma }^{(k)}=({{V}^{(k)}},{{E}^{(k)}})$, a subset of size $k$ of the edge set $E^{k}$, which must satisfy the following conditions:
\end{definition}
(1) For any two edges $e({{x}_{i}},{{x}_{j}}),\ e({{x}_{p}},{{x}_{q}})\in CM$, ${{x}_{i}}\cap {{x}_{p}}=\emptyset$ and ${{x}_{j}}\cap {{x}_{q}}=\emptyset$ should be hold; (2) $\left| CM \right|=\left| S_{\gamma }^{(k)} \right|=\left| S_{\mu }^{(k)} \right|$.
Note that for the graph $G_{\mu ,\gamma }^{(k)}$, there are $k!$ complete matching cases: $CM\in {{P}^{(k)}}=\{{{M}_{1}},{{M}_{2}},...,{{M}_{k!}}\}$.

\begin{definition}\textbf{Center Similarity Matrix}\label{Center Similarity Matrix}
	For any complete matching $PM\in {{P}^{(k)}}$ of the center bipartite graph $G_{\mu ,\gamma }^{(k)}$, the corresponding "center similarity matrix" with sparse property, denoted by $CD$, is defined as:
\end{definition}
\begin{equation} \label{EQ5}
CD_{[p,q]}^{(k)}=\left\{ \begin{aligned}
& \| {{x}_{p}}-{{x}_{q}}\|_{2}^{2},\ \ if\ e({{x}_{p}},{{x}_{q}})\in CM \\ 
& 0\ \ otherwise \\ 
\end{aligned} \right.
\end{equation} 


\begin{definition}\textbf{Center Pairwise Loss}\label{Center Pairwise Loss}
	Given a matching scheme $CM$, the expectation of center similarity matrix $CD$, called "center pairwise loss", is used to quantify the center pairwise relationships in this matching $CM$
\end{definition}
\begin{equation} \label{EQ6}
\begin{aligned}
L(CM)=\frac{1}{k}{{I}_{1*k}}({{CD}^{(k)}})*I_{1*k}^{T}
\end{aligned}
\end{equation}
where ${{I}_{1*k}}=(1,1,...,1)_{_{1*k}}^{T}$ is an auxiliary matrix for vectorization.
\begin{definition}\textbf{Center Pairwise Matching}\label{Center Pairwise Matching}
	Given the number of clusters k, there exist Ki complete matches between the mean center and the density center. We use the "center pairwise matching" as the matching scheme under the current number of clusters, expressed as follows.
\end{definition}
\begin{equation} \label{EQ7}
\begin{aligned}
CP{{M}^{(k)}}=\left\{ \left. \underset{e({{x}_{p}},{{x}_{q}})\in PM}{\mathop{\bigcup}}\,e({{x}_{p}},{{x}_{q}}) \right|\arg \underset{PM\in {{P}^{(k)}}}{\mathop{\min }}\,(L(PM)) \right\}
\end{aligned}
\end{equation}
It is worth mentioning that the loss value of the center pairwise matching is considered as the unique feature value of the current number of clusters during cluster number optimization. To this end, the optimal number of clusters should be projected as the global minimum center loss, thus we define

\begin{equation} \label{EQ8}
\begin{aligned}
{K^*}=\underset{k}{\mathop{\arg }}\,\min (L(CP{{M}^{(k)}}))
\end{aligned}
\end{equation}

\subsection{Low Confidence Sample Filtering}
Inspired by self-paced learning in which samples have different confidence levels for a given learning task, we note that some extreme points in the data space, especially at the edges of clusters, between clusters, or away from clusters, may disturb the cluster number identification. These points do not reveal the core structure of the clusters as well as tend to incur cluster merging, weaken the independence of clusters, eventually breaking the consistency between "the actual number of clusters," "the best clustering result," and "the best score," which is the main factor that current work is hard to deal with complex data. It is noteworthy that the existing methods are not aware of this fact. To this aim, we propose a boundary filtering method based on space vector decomposition, aiming to enhance the robustness of our method against complex data.

We find that the boundary object and its neighbors have strong unbalance in the projection subspace, forming the typical skew distribution in the direction of the base vector. By contrast, the neighbors of the core object show significant symmetry in the direction of the basis vector\cite{BPSVD}. With this in mind, we first treat the close relationship between target object and its neighbors as independent space vectors, described as follows. 

Let ${{x}_{i}}$ be any object in data set $X=\{{{x}_{1}},{{x}_{2}},...,{{x}_{n}}\}\in {\mathbb{R}^{m\times n}}$. According to the space vector decomposition theorem, there exists a unique ordered sequence of real number $\{{{\lambda }_{1}},{{\lambda }_{2}},...,{{\lambda }_{m}}\}$ such that ${{x}_{i}}={{\lambda }_{1}}{{e}_{1}}+{{\lambda }_{2}}{{e}_{2}}+...+{{\lambda }_{m}}{{e}_{m}}$, where $\{{{e}_{1}},{{e}_{2}},...,{{e}_{m}}\}$ denotes the basis vectors for the data space ${\mathbb{R}^{m}}$. For $ {\forall}$ ${{x}_{i}},{{x}_{j}}\in X$, the neighborhood vector formed by the two in ${\mathbb{R}^{m}}$ is given by
\begin{equation} \label{EQ9}
\begin{aligned}
{{h}_{i,j}}=({{\lambda }_{i,1}}-{{\lambda }_{j,1}},...,\ {{\lambda }_{i,d}}-{{\lambda }_{j,d}},...,{{\lambda }_{i,m}}-{{\lambda }_{j,m}}){{({{e}_{1}},...,{{e}_{m}})}^{T}}\ 1\le i,j\le n
\end{aligned}
\end{equation}
where ${{\lambda }_{i,d}}-{{\lambda }_{j,d}}$ denotes the projection of ${{h}_{i,j}}$in the direction of basis vector ${{e}_{d}}$.

Notably, the neighbors of ${{x}_{i}}$ refers to objects within the open neighborhood of $x_i$ with the radius as $dc$, i.e.,${{N}_{dc}}({{x}_{i}})$. As the typical operation to describe spatial position relations, the local representative vector ${\mathcal{H}_{i}}$ is customized as the vector addition regarding ${{x}_{i}}$ and its neighbors.
\begin{equation} \label{EQ10}
{\mathcal{H}_{i}}= 
\sum\limits_{{{x}_{j}}\in {{N}_{dc}}({{x}_{i}})}{({{h}_{i,j}})}\text{=}{{(\sum\limits_{{x}_{j}}{({{\lambda }_{i,1}}-{{\lambda }_{j,1}})},...,\sum\limits_{{x}_{j}}{({{\lambda }_{i,m}}-{{\lambda }_{j,m}}))(}{{e}_{1}},...,{{e}_{m}})}^{T}} 
\end{equation}
where $\sum{({{\lambda }_{i,d}}-{{\lambda }_{j,d}})}$ symbolizes the projection of the neighborhood vectors of $x_{i}$ in the direction of the basis vector ${{e}_{d}}$.  In this regard, we make the following derivation about projection of the neighborhood vector onto the feature space:

\begin{theorem}
	Given a one-dimensional feature space ${{v}_{d}}(1\le d\le m)$ and the neighborhood vector ${{h}_{i,j}}$, then the projection of the neighborhood vector ${{h}_{i,j}}$ onto it is ${{e}_{d}}{{(e_{d}^{\text{T}}{{e}_{d}})}^{-1}}e_{d}^{\text{T}}{{h}_{ij}}$, where ${{e}_{d}}$ is the basis vector of ${{v}_{d}}$.
	\label{thm-2}
\end{theorem}
\begin{proof}
	Assuming $\mathcal{S}$ is a subspace of $\mathbb{R}^m$ ($\mathcal{S} \in \mathbb{R}^p, s.t. p \le m $) and ${\{e_1^{'}},{e_2^{'}},{e_3^{'}},{e_4^{'}},...,{e_p^{'}}\}$ is a basis, then for any object  $a\in \mathcal{S}$, there is only $a={y_1{e_1}^{'}}+{y_2{e_2}^{'}}+{y_3{e_3}^{'}}+{y_4{e_4}^{'}}+...+{y_p{e_p}^{'}}$. Let matrix $A=[{e_1^{'}},{e_2^{'}},{e_3^{'}},{e_4^{'}},...,{e_p^{'}}]_{m\times p}$, then $a=Ay,y\in \mathbb{R}^p$.
	
	According to the above statement, let $Proj_sh_{i,j}$ denote the projection of neighborhood vector $h_{i,j}$ (composed of  ${{x}_{i}}\in {\mathbb{R}^{m}}$ and its neighbor ${{x}_{j}}\in {\mathbb{R}^{m}}$) on subspace $\mathcal{S}$, then $Proj_sh_{i,j}\in \mathcal{S}$ with
	\begin{equation}\label{EQ11} 
	Proj_sh_{i,j}=Ay,y\in R^p
	\end{equation}
	Based on the fundamental theory of linear algebra, the neighborhood vector $h_{i,j}$ is such that $h_{i,j}=Proj_{s^\bot}h_{i,j}+Proj_sh_{i,j}$, where $Proj_{s^\bot}h_{i,j}$ denote the projection of $h_{i,j}$ on the orthogonal complement of subspace $S$, which is equal to the null space of $A^T$. We have:
	\begin{align}\label{EQ12}
	Proj_{s^\bot}h_{i,j}&=h_{i,j}-Proj_sh_{i,j}\in Null(A^T) 
	\end{align}	
	
	Left-multiplying both sides by $A^{T}$ and simplifying, we obtain $A^{T} h_{i, j}-A^{T} A y=0$, then $y=\left(A^{T} A\right)^{-1} A^{T} h_{i, j}$. Substituting the expression for $y$ in the $Proj_sh_{i,j}$, we have $Proj_sh_{i,j}=Ay=A(A^TA)^{-1}A^Th_{i,j}$. Now, \cref{EQ11} can be rewritten as
	\begin{align}\label{EQ13}
	Proj_sh_{i,j}=Ay=A(A^TA)^{-1}A^T(x_i-x_j)
	\end{align}
	
	
	Without loss of generality, we assume that the subspace $\mathcal{S}$ is instantiated as a one-dimensional space $s_d$, thus the matrix $A$ could be transformed into a column vector $e_d$, i.e. the basis vector of this one-dimensional space. Consequently, projection is formalized as $e_d(e_d^Te_d)^{-1}e_d^Th_{ij}$. The theorem is proved. $\square$	
\end{proof}


With the help of orthogonality, we utilize a standard orthogonal basis ${{\{{{e}_{d}}\}}_{1\le d\le m}}$ ,($\ \text{s.t.} \langle e_{i}^{{}},{{e}_{j}}\text{=0}\rangle \cap e_{i}^{\text{T}}{{e}_{i}}\text{=}1$),to instantiate matrix $A$ in \cref{EQ13}. After that, the projections of the neighborhood vector $h_{i,j}$ composed of $x_i$ and $x_j$ on the directions of the $m$ orthogonal basis vectors can be expressed as:
\begin{equation}  \label{EQ14}
Proj_{v_d}h_{i,j}=(e_d(e_d^Te_d)^{-1}e_d^Th_{ij})=(e_de_d^Tx_i-e_de_d^Tx_j)=(x_{id}-x_{jd}),1\le d\le m
\end{equation}
Thus, local representative vector from \cref{EQ10} is formalized as follows:
\begin{align}\label{EQ15} 
\mathcal{H}_i & =(\sum_{j}{Proj_{v_1}h_{i,j}},\sum_{j}{Proj_{v_2}h_{i,j},...,\sum_{j}{Proj_{v_m}h_{i,j}}})(e_1,e_2,e_3,...,e_m)^T \nonumber \\
&=(\sum_{j}{e_1e_1^Tx_i-e_1e_1^Tx_j}, \sum_{j}{e_2e_2^Tx_i-e_2e_2^Tx_j},..., \sum_{j}{e_me_m^Tx_i-e_me_m^Tx_j})\left[ \begin{matrix}
1 & 0 & ... & 0  \\
0 & 1 & 0 & ...  \\
... & 0 & 1 & 0  \\
0 & ... & 0 & 1  \\
\end{matrix} \right] \nonumber \\ 
& =(\sum\limits_{j}{{{x}_{i1}}-{{x}_{j1}}},\ \sum\limits_{j}{{{x}_{i2}}-{{x}_{j2}}},...,\ \sum\limits_{j}{{{x}_{im}}-{{x}_{jm}}}) \nonumber \\
\end{align}

Mathematically, the length of the local representative vector $\mathcal H$ formed by the boundary object is significantly longer than that of the core object. in this case, the P-norm is preferred here.
\begin{definition}\textbf{Boundary degree}
	Considering the computational cost, the Boundary degree ${{\varphi}}$ is defined by \cref{EQ16} with $p$=1.
\end{definition}
\begin{equation} \label{EQ16} 
\begin{aligned}
\varphi_i=	||{{\mathcal{H}}_{i}}|{{|}_{p}}={{(\sum\limits_{d}{{{(|\sum\limits_{j}{{{x}_{id}}}-{{x}_{jd}}|)}^{p}}})}^{\frac{1}{p}}}\ \ \ \ s.t.\ 1\le d\ \le m,\ {{x}_{j}}\in {{N}_{dc}}({{x}_{i}})
\end{aligned}
\end{equation}
where the role of boundary degree ${{\varphi}}$ is to determine whether ${x_i}$ is a boundary point; a larger value suggests that $x_i$ should be removed first.

After that, the dataset $X$ can be evolved into a core subset:
\begin{equation}\label{EQ17}
{{X}^{\prime}}=\{{{x}_{i}}\in X: {\varphi_i} \le A_{\varphi}[\left\lfloor n*\lambda \right\rfloor ]\}
\end{equation}
where $\lambda$ denotes the proportion of low confidence sample in $X$ and $A$ is a descending list of $\varphi$. $\lambda$ is empirically set to 10\%. Subsequently, the density centers and mean centers are derived from $X^{\prime}$ by CNMBI, which finally predicts the number of clusters by the proposed center pairwise matching. As summarized in Algorithm 1.

\begin{algorithm}[h]
	\DontPrintSemicolon
	\SetAlgoLined
	\KwIn {Input data: $X$}
	\KwOut {the optimal number $K^*$}
	Calculate Boundary degree of each sample in $X$, $\varphi_i$ using \cref{EQ16} ;\\
	Remove the low-confidence samples to form the core set $X^{\prime}$ of $X$ using \cref{EQ17}; \\
	\For{$k={{\hat{k}}_{\min }}:{{\hat{k}}_{\max }}$}{ 
		Find the density center set $S_{\gamma }^{(k)}$ using  \cref{EQ4} \\
		Find the mean center set $S_{\mu }^{(k)}$ using \cref{EQ1} \\
		Construct the Center Bipartite Graph $G_{\mu ,\gamma }^{(k)}$ using \cref{Center Bipartite Graph}\\
		Search the Center Complete Matching $CM$ and calculate the Center Pairwise Loss $L(CM)$ using \cref{Center Complete Matching} and \cref{EQ7} \\
		Determine the Pairwise Matching $CPM^{(k)}$ under the $k$ using \cref{EQ7}
	} 
	$K^*={\mathop{\arg }}\,\min (L(CP{{M}^{(k)}})), {k \in [{\hat{k}}_{\min },{\hat{k}}_{\max }]}$ \\
	\KwResult{$K^*$ }
	\caption{CNMBI}
	\label{Algorithm2}
\end{algorithm}



\section{Experiments}

\textbf{Data Sets.} To demonstrate the effectiveness of the proposed CNMBI, a variety of data scenarios are covered, the information of which is presented in \cref{Table1}. The 8 synthetic datasets are generated from Shape sets of the Clustering basic benchmark\footnote{http://cs.joensuu.fi/sipu/datasets/}, and 3 high-dimensional datasets are obtained from the UCI \cite{UCIrepository}. The image datasets (CIFAR-10\footnote{http://www.cs.toronto.edu/ kriz/cifar.html}, STL-10\footnote{https://cs.stanford.edu/ acoates/stl10/}, MNIST\footnote{http://yann.lecun.com/exdb/mnist/}, Pendigits, Optical Recognition, and Pointing\footnote{http://www-prima.inrialpes.fr/Pointing04}) are considered. \textbf{Evaluation Metrics.} We record the optimal number of clusters (NC) and calculate the accuracy metric (ACC: i.e., the proportion of successful estimates in 50 runs). \textbf{Experimental Settings.} We set the number of clusters $k$ to the acknowledged range of 2 to $\sqrt n$, performing 50 runs. We compare the proposed CNMBI with 4 state-of-the-art approaches, LCCV\cite{LCCV}, VCIM\cite{VCIM}, DPC-AHS\cite{DPC-AHS}, CNAK\cite{CNAK} and well-known validity index SC\cite{Silhouettes}.

\begin{table}[htp]
	\caption{Basic information about the studied data sets.}
	\label{Table1}
	\renewcommand\tabcolsep{3pt}
	\renewcommand\arraystretch{0.9}
	\footnotesize
	\centerline{
		\begin{tabular}{cccccc}
			\toprule
			Name                                        & Size  & Dimension         & Class      & Description                                      \\
			\midrule
			Unbalanced                                     & 1500  & 2         & 3          & Multi-density                         \\
			Flame                                       & 240   & 2         & 2          & Overlapping                                      \\
			Jain                                        & 373   & 2         & 2          & Crescent, Multi-density                          \\
			S2                                          & 5000  & 2         & 15         & Sub-cluster,Overlapping                          \\
			A1                                       & 3000  & 2         & 20         & Small cluster                                      \\
			Asymmetric                                  & 1000  & 2         & 5          & Unbalanced, noisy                                 \\
			Heterogeneous                                        & 400   & 2         & 3          & Heterogeneous geometric\\
			Multi-objective                                       & 1000  & 2         & 4          & elongated, circle-like,   multi-objective        \\
			\midrule
			REUTERS-10K         & 10000 & 12000     & 4          & Word, News, Text,                                \\
			Covid-19            & 171   & 53        & 3          & Covid-19                                         \\
			RNA-seq             & 801   & 20531     & 5          & Gene expression, Nonlinear                       \\
			\midrule
			MNIST               & 10000 & 28*28     & 10         & OCR, high-dimensional                            \\
			Pendigits           & 10992 & 4*4       & 10         & Handwritten Digits                               \\
			Optical Recognition & 5620  & 64        & 10         & OCR, Handwritten Digits                          \\
			Pointing Data       & 1395  & 384*284*3 & 15         & Posture recognition    \\
			STL-10              & 13000 & 3*96*96   & 10         & Complex Contextual  \\
			CIFAR-10            & 10000 & 3*32*32   & 10         & Complex Contextual, Large-Scale \\
			\bottomrule
		\end{tabular}
	}
\end{table}


\subsection{Quantitative Study}
\cref{Table2} summarizes the experimental results of the quantitative study, from which we can draw some conclusions. First, our algorithm correctly identified 13 out of 17 datasets in 50 runs, 8 more than the second place, achieving top-1 performance. Second, our method is suitable for data containing arbitrary-shaped clusters. For instance, the clusters in datasets "S2"(\cref{Fig3-Synthetic}(a)), "A1"(\cref{Fig3-Synthetic}(b)), "Flame"(\cref{Fig3-Synthetic}(d)), and "Asymmetric"(\cref{Fig3-Synthetic}(f)) almost overlap visually. Whereas datasets "Unbalanced"(\cref{Fig3-Synthetic}(c)) and "Jain"(\cref{Fig3-Synthetic}(e)) are examples of multi-density, showing manifold and unbalanced, respectively.  As a result, only our method accurately identified all of them.  Besides, \cref{Table2} also reports the more challenging datasets from real world. For example, REUTERS-10K is a typical text type; Covid-19 records biochemical indicators of 3 types of COVID-19 with significant dimensional-difference; RNA-seq records a surprising 20,000 dimensions of gene transcription data. The six large-scale image datasets further validate the effectiveness of our algorithm CNMBI, especially for complex CIFAR-10 and STL-10, where only CNMBI achieves the correct cluster number estimation. The above satisfactory performance is mainly thanks to the joint modeling of the center pairwise matching scheme and the boundary filtering principle, the former ensuring the flexibility of our method in terms of data dimension and shape, while the latter enhancing the robustness of CNMBI when handling complex data. It is also worth mentioning that the scores under the ACC metric indicate that our method is more stable than other methods over multiple runs. In particular, \cref{Fig-Loss} visualizes the optimization process of part of the data set.





\subsection{Ablation Study}

\cref{Table4} reports the experimental results, from which we can draw some conclusions: 1) Method-1 removes the boundary filtering strategy compared to our method, resulting in a significant decrease in $NC$ score from a perfect score of 1 to 0.2 on the Pointing Dataset, which demonstrates that this strategy plays an important role in improving robustness; 2) The results of Method-5 indicate that our method has good scalability, i.e., cluster number identification can also be achieved by comparing the mean centers generated twice, although it is not as stable as our original method, as in the case of $NC$ score.

\begin{table}[H]
	\caption{Cluster number prediction on challenging datasets over 50 runs. }
	\label{Table2}
	\renewcommand\tabcolsep{5pt}
	\renewcommand\arraystretch{1}
	\footnotesize
	\centerline{
		\begin{tabular}{ccccccccccccc}
			\toprule
			\multirow{2}{*}{Dataset} & \multicolumn{2}{c}{SC}  & \multicolumn{2}{c}{DPC-AHS} & \multicolumn{2}{c}{VCIM} & \multicolumn{2}{c}{LCCV}  & \multicolumn{2}{c}{CNAK}  & \multicolumn{2}{c}{CNMBI} \\
			\cmidrule(lr){2-3}  	\cmidrule(lr){4-5} \cmidrule(lr){6-7} \cmidrule(lr){8-9} \cmidrule(lr){10-11} \cmidrule(lr){12-13}
			& NC         & ACC        & NC         & ACC         & NC         & ACC & NC         & ACC & NC         & ACC & NC         & ACC     \\
			\midrule
			Unbalanced    & 2           & 0          & 2           & 0          & 2           & 0          & 2                         & 0          & \textbf{3}  & 0.9        & \textbf{3}  & \textbf{1} \\
			Flame      & 4           & 0          & \textbf{2}  & \textbf{1} & 4           & 0          & 4                         & 0          & 1           & 0          & \textbf{2}  & \textbf{1} \\
			Jain       & \textbf{2}  & 0.5        & \textbf{2}  & \textbf{1} & 3           & 0          & 7                         & 0          & 3           & 0          & \textbf{2}  & \textbf{1} \\
			S2         & \textbf{15} & 0.4        & 13          & 0          & 2           & 0          & 23                        & 0          & \textbf{15} & \textbf{1} & \textbf{15} & \textbf{1} \\
			A1      & \textbf{20} & 0.1        & \textbf{20} & \textbf{1} & 2           & 0          & 4                         & 0          & \textbf{20} & \textbf{1} & \textbf{20} & \textbf{1} \\
			Asymmetric & 2           & 0          & \textbf{5}  & \textbf{1} & 2           & 0          & \textbf{5}                & \textbf{1} & \textbf{5}  & 0.5        & \textbf{5}  & \textbf{1} \\
			Heterogeneous       & 2           & 0          & 2           & 0          & 5           & 0          & 4                         & 0          & 2           & 0          & \textbf{3}  & \textbf{1} \\
			Multi-objective      & 12          & 0          & 2           & 0          & 3           & 0          & \textbf{4}                & \textbf{1} & 3           & 0          & \textbf{4}  & \textbf{1} \\	
			REUTERS-10K             & 78          & 0            & \textbf{4}  & \textbf{1} & 2          & 0            & 92          & 0          & 1           & 0            & \textbf{4}                                   &  0.1            \\
			Covid-19            & 2           & 0            & \textbf{3}  & \textbf{1} & \textbf{3} & \textbf{1}   & 2           & 0          & \textbf{3}  & \textbf{1}   & \textbf{3}                          & \textbf{1}   \\
			RNA-seq             & \textbf{5}  & 0.6 & \textbf{5}  & \textbf{1} & \textbf{5} & \textbf{1}   & \textbf{5}  & \textbf{1} & \textbf{5}  & 0.6 & \textbf{5}                          & \textbf{1}   \\
			MNIST               & \textbf{10} & 0.1 & 8           & 0          & 3          & 0            & \textbf{10} & \textbf{1} & \textbf{10} & 0.1 & \textbf{10}                         & \textbf{1}   \\
			Pendigits           & 15          & 0            & 13          & 0          & 2          & 0            & 13          & 0          & 11          & 0            & \textbf{10}                         & 0.8 \\
			Optical  & \textbf{10} & 0.7 & 5           & 0          & 7          & 0            & \textbf{10} & \textbf{1} & \textbf{10} & \textbf{1}   & \textbf{10}                         & \textbf{1}   \\
			Pointing Data       & 14          & 0            & 25          & 0          & 3          & 0            & 24          & 0          & 9           & 0            & \textbf{15} & \textbf{1}	 \\
			STL-10       & 20          & 0            & 1          & 0          & -          & -            & -          & -          & 1           & 0            & \textbf{10} & 0.2	 \\
			CIFAR-10       & 1          & 0            & 1          & 0          & -          & -            & 56          & 0          & \textbf{10}           & 0.1            & \textbf{10} & 0.3	 \\
			\bottomrule 
		\end{tabular}
	}
\end{table}

\vspace{-1cm}
\begin{figure}[H]
	\centerline{
		\subfloat[S2]	
		{\includegraphics[width=0.25\linewidth]{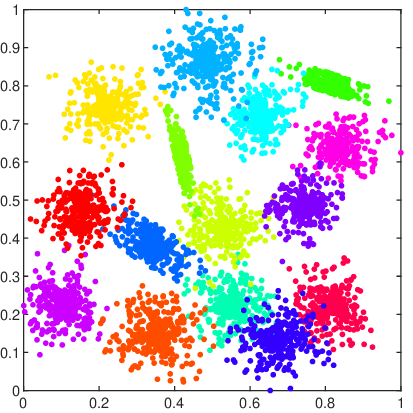}\label{S2}}
		\subfloat[A1]
		{\includegraphics[width=0.25\linewidth]{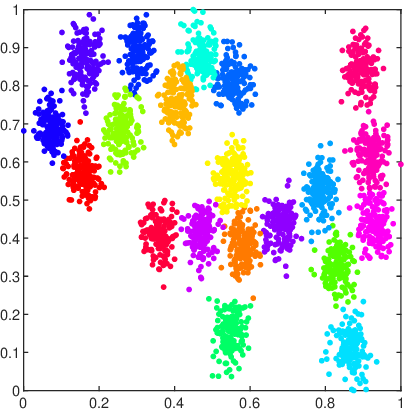}\label{A1}} 
		\subfloat[Unbalanced]	
		{\includegraphics[width=0.25\linewidth]{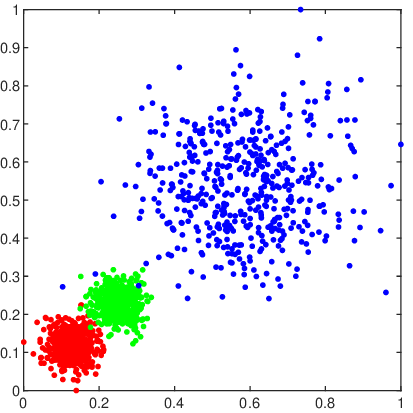}\label{Unbalanced}}	
		\subfloat[Flame]
		{\includegraphics[width=0.25\linewidth]{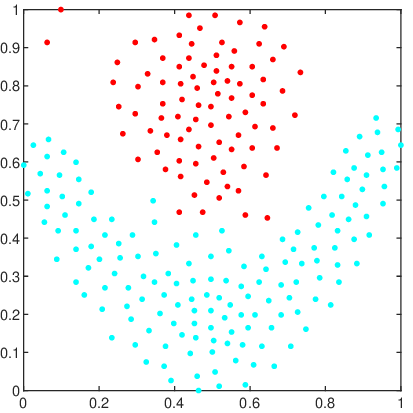}\label{Flame}}} 
	\vspace{-4mm}
	\centerline{
		\subfloat[Jain]
		{\includegraphics[width=0.25\linewidth]{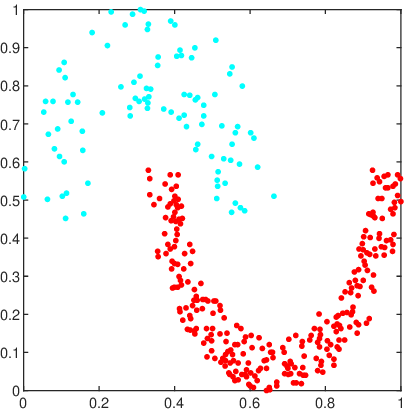}\label{Jain}}
		\subfloat[Asymmetric]
		{\includegraphics[width=0.25\linewidth]{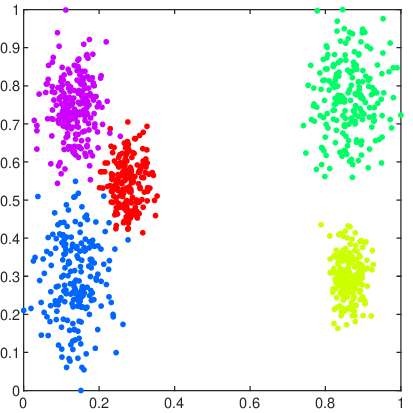}\label{Asymmetric}} 
		\subfloat[Heterogeneous]	
		{\includegraphics[width=0.25\linewidth]{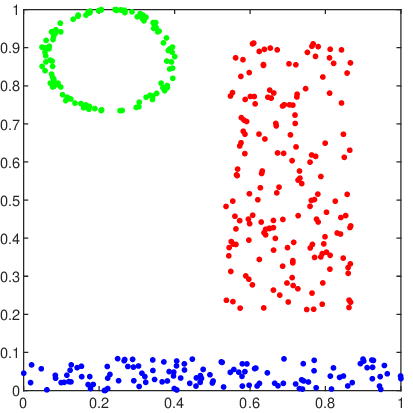}\label{Heterogeneous}} 
		\subfloat[Multi-objective]
		{\includegraphics[width=0.25\linewidth]{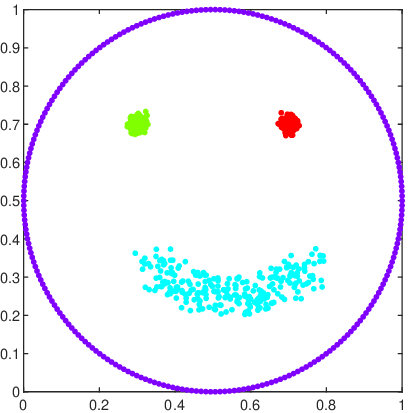}\label{Multi-objective}}}
	\setlength{\abovecaptionskip}{0.5mm} 
	\caption{Synthetic dataset containing arbitrary-shaped clusters.} 
	\label{Fig3-Synthetic}
\end{figure}


\begin{figure}[htp]
	\centering
	\subfloat[MNIST]		
	{\includegraphics[width=0.33\linewidth]{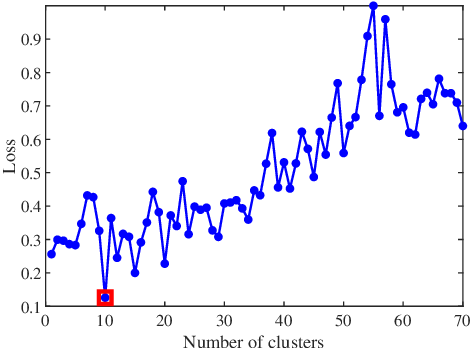}}
	\subfloat[STL-10]
	{\includegraphics[width=0.33\linewidth]{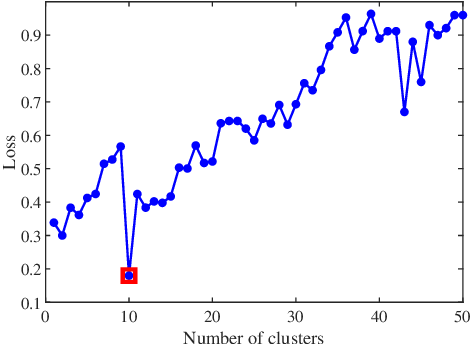}}
	\subfloat[Pointing Data]
	{\includegraphics[width=0.33\linewidth]{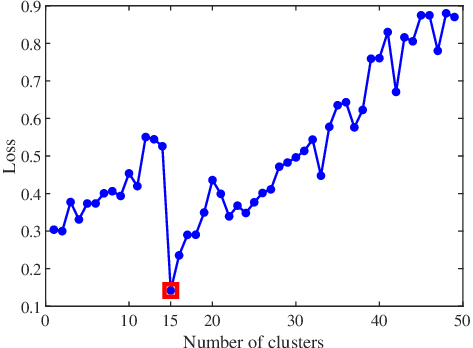}}
	\vspace{-2mm}
	\caption{Plots of $k$ and center pairwise loss on partial image data sets. }
	\label{Fig-Loss}
\end{figure}

\vspace{-0.35cm}
\begin{table}
	\caption{Ablation study of CNMBI and its
		degradation methods. $S_{\mu}$, $S_{\gamma}$ and ${{\varphi}}$ denote the Mean center, Denisty center, and boundary filtering, respectively. }
	\label{Table4}
	\renewcommand\tabcolsep{9pt}
	\renewcommand\arraystretch{0.9}
	  \footnotesize
	\centering
	\begin{tabular}{cccccc}
		\toprule
		Methods & $S_{\mu}$ & $S_{\gamma}$ & ${{\varphi}}$   & MNIST   & Pointing Data    \\
		\midrule
		Ours & $\checkmark$ & $\checkmark$ & $\checkmark$ & NC: 10 \; ACC: 1	& NC: 15 \; ACC: 1\\
		\midrule
		Method-1 & $\checkmark$ & $\checkmark$ & $\times$ & NC: 10 \; ACC: 0.6 & NC: 15 \; ACC: 0.2\\
		\midrule
		Method-2 & $\checkmark$ & $\times$ & $\times$& NC: 1 \; ACC: 0	& NC: 1 \; ACC: 0\\
		\midrule
		Method-3 & $\times$ & $\checkmark$ & $\times$& NC: 10 \; ACC: 0.3	& NC: 15 \; ACC: 0.1\\
		\midrule
		Method-4 & $\checkmark$ & $\times$ & $\checkmark$ & NC: 1 \; ACC: 0	& NC: 1 \; ACC: 0\\
		\midrule
		Method-5 & $\times$ & $\checkmark$ & $\checkmark$ & NC: 10 \; ACC: 0.5	& NC: 15 \; ACC: 0.4\\
		
		\bottomrule
	\end{tabular}
\end{table}

\vspace{-0.35cm}  

\subsection{Parameter Sensitivity}
CNMBI is parameter-free since the parameters, cutoff distance $dc$ in density center and proportion $\lambda$ in boundary filtering, are dynamically adaptive. Here we test CNMBI in three typical scenarios, each with four datasets, as reported in \cref{Fig-Robustness}. For scenario "Noise," different levels of noise interference are added. The "Density" consists of 8 clusters with unbalanced density. For "Number," the four data sets contained have different numbers of clusters, 5,10,20,40.

The results are concentrated in \cref{Fig-Robustness}(m). Concretely, as given by the red line in \cref{Fig-Robustness}, it can be found that from the first to the fourth dataset, CNMBI behaves satisfactorily without error despite the increase in the number of clusters. The Noise-40 and Noise-50 dataset (\cref{Fig-Robustness}(g),(h)) cannot be visually separated due to intense noise disturbance, CNMBI can still determine the desired number of clusters 2 (see the blue-green line segment in \cref{Fig-Robustness}). Moreover, the density imbalance gradually increases from Density-1 to Density-4, and the green line segment in \cref{Fig-Robustness} shows that CNMBI is robust to density changes. 
Generally, benefiting from the consideration of active boundary filtering and simple but effective cluster number features generated by center matching, CNMBI has better robustness against data with complex distribution than other methods.

\subsection{Boundary Pattern Recognition}
Our method can extract valuable boundary pattern information while determining the number of clusters. Although some points cause undesirable results in cluster number determination, in practice, these objects often possess the characteristics of two or more groups simultaneously and have great potential value, such as populations carrying viruses but not manifested in epidemiology and scribbled characters in optical character recognition.
\vspace{-4mm}
\begin{figure}[htp]
	\centering
	\subfloat[Density-1]
	{\includegraphics[width=0.19\linewidth]{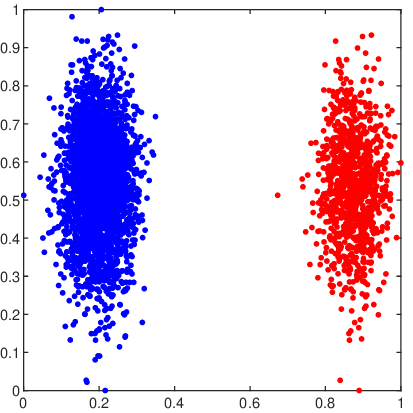}}
	\subfloat[Density-2]
	{\includegraphics[width=0.19\linewidth]{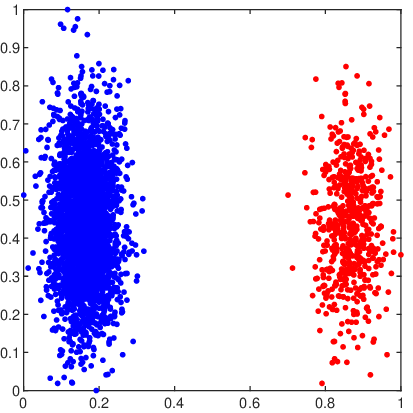}}
	\subfloat[Density-3]
	{\includegraphics[width=0.19\linewidth]{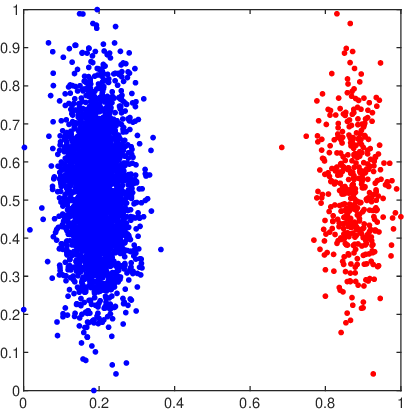}}	
	\subfloat[Density-4]
	{\includegraphics[width=0.19\linewidth]{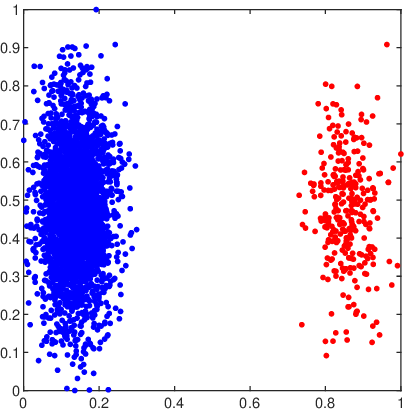}}
	\vspace{-4mm}
	\subfloat[Noise-20]
	{\includegraphics[width=0.19\linewidth]{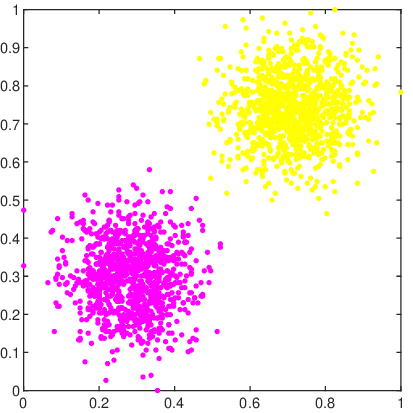}}
	\subfloat[Noise-30]
	{\includegraphics[width=0.19\linewidth]{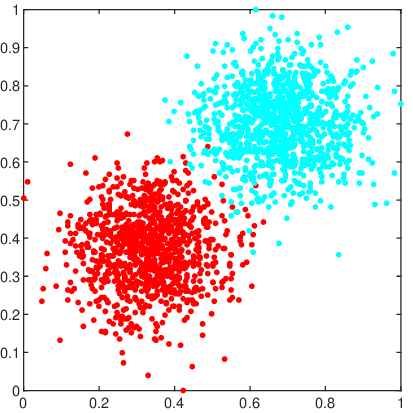}}
	\subfloat[Noise-40]
	{\includegraphics[width=0.19\linewidth]{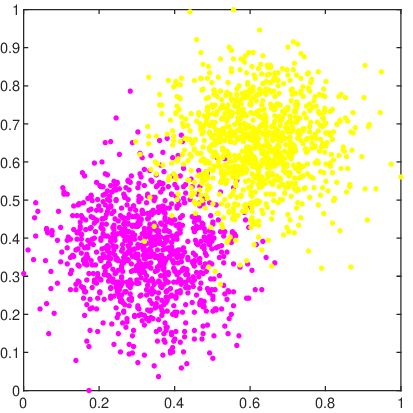}}
	\subfloat[Noise-50]	
	{\includegraphics[width=0.19\linewidth]{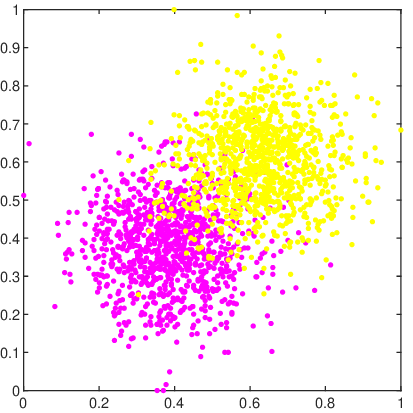}}
	\vspace{-4mm}
	\subfloat[Num-5]
	{\includegraphics[width=0.19\linewidth]{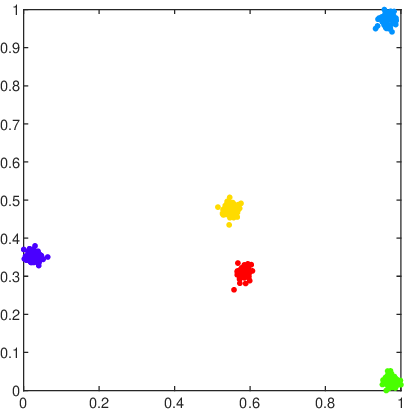}}
	\subfloat[Num-10]
	{\includegraphics[width=0.19\linewidth]{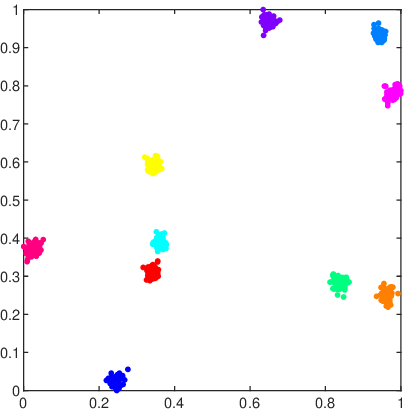}}
	\subfloat[Num-20]
	{\includegraphics[width=0.19\linewidth]{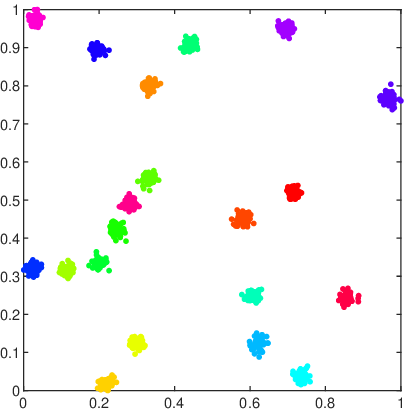}}
	\subfloat[Num-40]	
	{\includegraphics[width=0.19\linewidth]{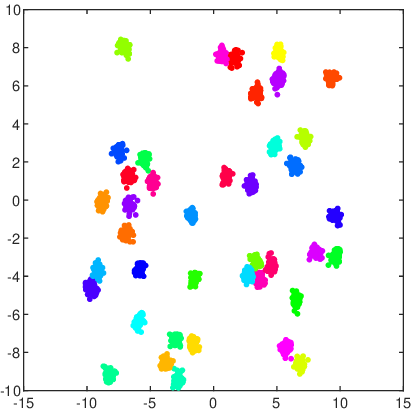}}	\\
	{\includegraphics[width=0.8\linewidth]{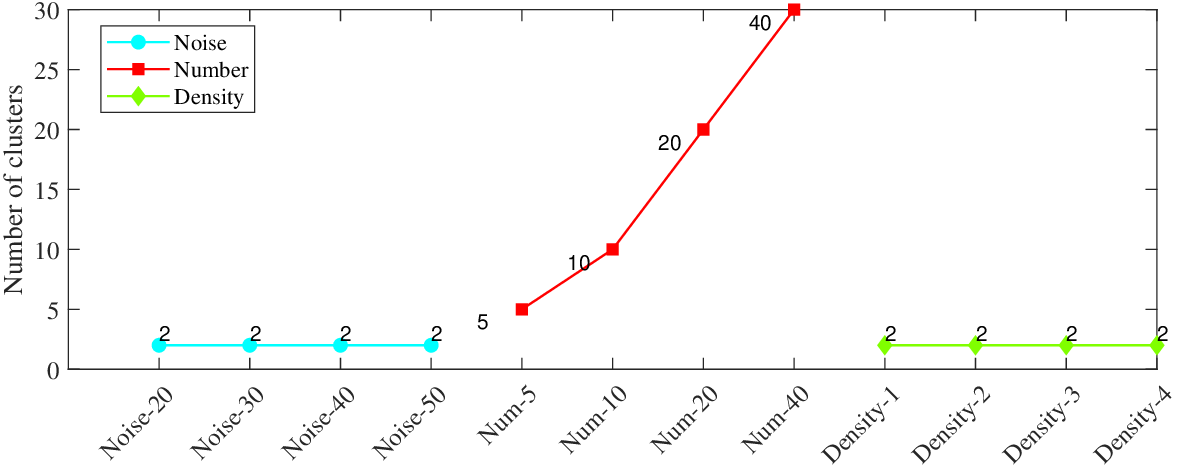}}\\			
	\setlength{\abovecaptionskip}{-0.1mm} 	
	\caption{Three extreme scenarios in Robustness }
	\label{Fig-Robustness}
\end{figure}
\vspace{-4mm}

\cref{Fig-Boundary}(a) shows the boundary samples we extracted from the handwriting dataset (MNIST), from which it can be seen that the characters are illegible, with blurred, hollow, continuous strokes that are distinctly different from normal fonts. The pointing dataset records 23 types of head posture images from 15 volunteers. As shown in \cref{Fig-Boundary}(b), the identified images have exaggerated movements and complex expressions from Pointing Data, which is beneficial for determining the amplitude threshold of each action in the pose recognition. Moreover, our proposed boundary filtering method is informative for other learning tasks.
\vspace{-7mm}
\begin{figure}[htp]
	\centering
	\subfloat[Illegible handwritten characters.]	
	{\includegraphics[width=0.43\linewidth]{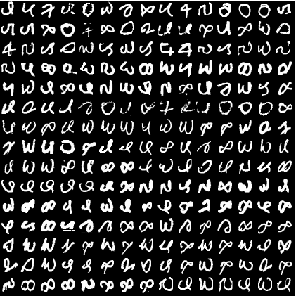}}
	\subfloat[Complex gestures.]	
	{\includegraphics[width=0.43\linewidth]{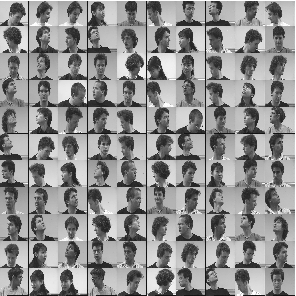}}
	\vspace{-1mm}
	\caption{Boundary pattern information extraction.}
	\label{Fig-Boundary}
\end{figure}
\vspace{-7mm}



\section{Discussion}
\label{sec:discussion}

In this section, we provide a deeper introspection into the structural mechanics of CNMBI and explore its conceptual evolution towards broader learning paradigms.

\subsection{Introspection of CNMBI Mechanics}
\label{subsec:introspection}

\noindent \textbf{Breaking Isotropic Assumptions via Dual-Center Coupling.} 
The fundamental limitation of conventional cluster validity indices is their implicit reliance on the assumption of isotropic, convex data distributions. Most methods evaluate the "goodness" of a partition by measuring inter-cluster separation and intra-cluster compactness. However, real-world data manifolds are often non-convex and characterized by varying densities, rendering distance-based compactness uninformative. 

CNMBI moves beyond this by introducing a \textit{Dual-Center Coupling} mechanism. Instead of scoring the final partition, we examine the structural alignment between the \textit{density center} (representing the topological core) and the \textit{mean center} (representing the geometric centroid).

We observe that in a suboptimal clustering state ($k \neq K$), the geometric centroid is easily skewed by partial cluster merging or splitting, while the topological core remains relatively stable. The "Center Pairwise Loss" essentially quantifies this structural tension. By seeking the global minimum of this loss, CNMBI effectively identifies the number of clusters where the geometric and topological perspectives reach a state of equilibrium, allowing for a robust determination of $K$ without a priori assumptions about the cluster shape.

\noindent \textbf{Boundary Filtering as Pattern Purification.} 
A critical bottleneck in unsupervised number determination is "centroid drift" caused by low-confidence samples located at the fringes of clusters or within overlapping regions. In high-dimensional spaces, these boundary points act as "noise" that blurs the distinctiveness of clusters. 

CNMBI addresses this via a boundary filtering module grounded in \textit{Space Vector Decomposition}. Unlike standard outlier detection, which often focuses on global distance, our approach evaluates the \textit{positional belief} of a sample based on the asymmetry of its local neighborhood vector. By projecting these vectors onto an orthogonal basis, we can identify samples that exist in high-entropy transition zones. Filtering these samples is not merely a denoising step but a process of \textbf{Pattern Purification}. By isolating the "cluster skeleton," we ensure that the subsequent center matching is conducted on the most stable and representative core of the data, significantly enhancing the precision of the loss function.

\noindent \textbf{Metric Generalization and Manifold Scalability.} 
From a structural standpoint, the pairwise matching framework in CNMBI is not strictly bound to Euclidean geometry. The algorithm's architecture suggests a generalized framework where the similarity between the dual centers can be redefined through diverse metric spaces:
\begin{itemize}
    \item \textbf{Manifold Consistency:} The matching process can be viewed as an attempt to find the point where the local manifold density peak aligns with the global cluster mass. This suggests that the framework is naturally scalable to non-linear manifolds where traditional $L_2$ norms may fail.
    \item \textbf{Metric Flexibility:} The "Center Pairwise Loss" serves as a meta-objective that can incorporate different distance or similarity measures depending on the data modality. By focusing on the \textit{behavioral behavior} of centers rather than the \textit{assignment of every single point}, CNMBI provides a computationally efficient and theoretically flexible pathway for identifying structures in ultra-high-dimensional or distorted feature spaces.
\end{itemize}

\subsection{Evolution Toward Open-World Discovery and Cross-Domain Applications}
\label{subsec:openworld_future}

\noindent \textbf{Paving the Way for Open-World Perception.} 
The robust estimation of cluster counts is a transformative cornerstone for transitioning from static clustering to the dynamic Open-World paradigm. A major bottleneck in Generalized Category Discovery \cite{zheng2024textual} and On-the-fly Category Discovery \cite{zheng2024prototypical} is the unrealistic assumption that the number of novel or fine-grained categories is known a priori. By synergizing CNMBI with Transformer-based global dependency modeling \cite{zhang2023tdec}, curriculum-driven density core assignments \cite{zheng2024deep}, and multi-modality co-teaching \cite{zheng2024textual}, future architectures can achieve a more holistic and autonomous understanding of evolving data streams without human-in-the-loop parameter tuning. Beyond discovery, this autonomous capability extends naturally to data-efficient learning and generative modeling. By dynamically capturing the underlying number of semantic modes in complex distributions, the proposed logic can provide critical structural guidance for adaptive dataset quantization \cite{LiZDXQ25} and diverse dataset distillation frameworks \cite{libeyond,li2026fixed,li2026efficient}, ensuring that synthetic or compressed proxies preserve the full topological and categorical diversity of the original data. Similarly, recognizing the true scale of latent cluster densities can enhance data-free knowledge distillation by anchoring diverse diffusion-based augmentations \cite{LiZ0XLQ24} to precise semantic centers. Finally, moving towards advanced 3D generative vision, these density-aware structural priors hold significant potential to refine probability density flow matching, enabling more stable and mode-preserving geodesic estimations for novel view synthesis \cite{wang2026geodesicnvs}.

\noindent \textbf{Cross-Domain Potential and Future Trajectories.}
CNMBI effectively addresses the pre-definition limitation of classical clustering algorithms by automating the estimation of cluster counts, thereby serving as a robust initialization engine for diverse downstream pipelines. For medical imaging, it facilitates pathological region partitioning by automatically determining the number of distinct tissue layers or lesion types in OCT images, which is essential for guided denoising and segmentation \cite{li2026cross,li2025retidiff,li2024efficient,li2026garnet}. In 3D vision, CNMBI enhances global solvers and vanishing point estimation \cite{zhao2026advances,zhao2024balf,zhao2023benchmark,liao2025convex} by discovering the true number of geometric directions in unstructured environments, and streamlines instance-level point cloud grouping \cite{qu2024conditional,qu2025end,qu2025robust}. For robotics, it automates the identification of latent physical modes (e.g., estimating the number of slip, rotation, or stable contact states) within tactile or sensor data streams, thereby refining stability analysis and cooperative control \cite{yan2025pandas,ruan2024q,yan2023stability,zhang2025ccma,xu2025sedm,tian2025measuring}. In embodied AI and video-language modeling, CNMBI aids in temporal event partitioning by clustering video segments into an adaptive number of semantic events \cite{li2025lion}, and assists in defining the optimal granularity for action primitives in VLA models \cite{li2025cogvla}, establishing the structured semantic groundwork necessary for deliberate System-2 reasoning capabilities \cite{song2025hume}. Furthermore, it optimizes knowledge distillation by providing the precise cluster scale for location-aware semantic masking \cite{lan2025acam,lan2026clockdistill,lan2024gradient}, and yields aligned structural counts that can effectively bridge semi-supervised domain translation via diffusion models \cite{wang2025ladb}. Finally, CNMBI provides a scalable way to estimate entity counts in open-vocabulary marine or remote sensing segmentation \cite{li2025exploring,li2025stitchfusion,li2025maris,li2025exploring2} and identifies structural group counts in socio-meteorological analysis \cite{liu2026health,shen2025aienhanced,shen2026mftformer}. In this paper, we propose a novel method, CNMBI, to identify the number of clusters, which is capable of being applied to some challenging data such as large-scale images and high-dimensional real-world data. We are the first to report the results on STL-10 and CIFAR-10. We do not rely on the clustering results of the entire data but instead utilize a few but highly representative cluster centers. This purposely straightforward contrast allows our method to be more scalable and flexible in dimension and shape, most importantly, sufficient for discovering actual cluster numbers through the positional behavior of the cluster center. Additionally, the principle of filtering low-confidence samples further enhances the robustness of our method, which is the first attempt in this field.

\subsubsection{Acknowledgements} This work was supported in part by the Shenzhen Fundamental Research Fund (JCYJ20210324132212030) and the Guangdong Provincial Key Laboratory of Novel Security Intelligence Technologies (2022B1212010005).

%
%
%
\bibliographystyle{splncs04}
\bibliography{mybibfile} 
%
%
%
%
%
\end{document}